\documentclass{llncs}

\usepackage{llncsdoc}
\usepackage{amsmath,amssymb,amstext}
\usepackage{graphicx}
\usepackage{algorithm, algorithmicx, algpseudocode}

\usepackage{amsthm}

\newcommand{\mc}{\mathcal}
\newcommand{\mb}{\mathbf}
\newcommand{\bb}{\mathbb}

\newcommand{\ov}{\overline}
\newcommand{\wt}{\widetilde}

\begin{document}
\title{Generalization and Robustness of Batched Weighted Average Algorithm with V-geometrically Ergodic Markov Data}

\author{Nguyen Viet Cuong\inst{1} \and Lam Si Tung Ho\inst{2} \and Vu Dinh\inst{3}}

\institute{Department of Computer Science, National University of Singapore, 117417, Singapore, \email{nvcuong@comp.nus.edu.sg}
\and
Department of Statistics, University of Wisconsin-Madison, WI 53706, USA, \email{lamho@stat.wisc.edu}
\and
Department of Mathematics, Purdue University, IN 47907, USA, \email{vdinh@math.purdue.edu}}

\maketitle

\begin{abstract}
We analyze the generalization and robustness of the batched weighted average algorithm for V-geometrically ergodic Markov data. This algorithm is a good alternative to the empirical risk minimization algorithm when the latter suffers from overfitting or when optimizing the empirical risk is hard. For the generalization of the algorithm, we prove a PAC-style bound on the training sample size for the expected $L_1$-loss to converge to the optimal loss when training data are V-geometrically ergodic Markov chains. For the robustness, we show that if the training target variable's values contain bounded noise, then the generalization bound of the algorithm deviates at most by the range of the noise. Our results can be applied to the regression problem, the classification problem, and the case where there exists an unknown deterministic target hypothesis.
\end{abstract}

\section{Introduction}
The generalization ability of learning algorithms has been studied extensively in statistical learning theory \cite{vapnik1998statistical}. One main assumption in traditional learning theory when studying this problem is that data, drawn from an unknown distribution, are independent and identically distributed (IID) \cite{valiant1984theory}. Although this assumption is useful for proving theoretical results, it may not hold in applications such as speech recognition or market prediction where data are usually temporal in nature \cite{steinwart2009learning}.

One attempt to relax this IID data assumption is to consider cases where training data form a Markov chain with certain mixing properties. A common algorithm that has been analyzed is the empirical risk minimization (ERM) algorithm, which tries to find the hypothesis minimizing the empirical loss on the training data. Generalization bounds of this well-known algorithm were proven for exponentially strongly mixing data \cite{zou2009generalization}, uniformly ergodic data \cite{zou2009learning}, and V-geometrically ergodic data \cite{zou2012}.

In this paper, we investigate another learning algorithm, the \emph{batched weighted average} (BWA) algorithm, when training data form a V-geometrically ergodic Markov chain. This algorithm is a batch version of the online weighted average algorithm with $L_1$-loss \cite{kivinen1999averaging}. Given the training data and a set of real-valued hypotheses, the BWA algorithm learns the weight of each hypothesis based on its prediction on the training data. During testing, the algorithm makes prediction based on the weighted average prediction of all the hypotheses on the testing data.

An advantage of the BWA algorithm when compared to the ERM algorithm is that the former may be less suffered from overfitting when the hypothesis space is large or complex \cite{freund2004generalization,freund2001averaging}. The BWA algorithm is also a good alternative to the ERM algorithm in cases where optimizing the empirical risk is hard.

We prove the generalization of the BWA algorithm by providing a PAC-style bound on the training sample size for the expected $L_1$-loss of the algorithm to converge to the optimal loss with high probability, assuming that training data are V-geometrically ergodic. The main idea of our proof is to bound the normalized weights of all the bad hypotheses whose expected loss is far from the optimal. This idea comes from the observation that when more training data are seen, the normalized weights of the bad hypotheses will eventually be dominated by those of the better hypotheses.

Using the same proof technique, we then prove the robustness of the BWA algorithm when training data form a V-geometrically ergodic Markov chain with noise. By robustness, we mean the ability of an algorithm to generalize when there is a small amount of noise in the training data. For the BWA algorithm, we show that if the training values of the target variable are allowed to contain bounded noise, then the generalization bound of the algorithm deviates at most by the range of the noise.

Our main results are proven mainly for the regression problem and the case where the pairs of observation and target variables' values are V-geometrically ergodic. However, we also give two lemmas to show that the results can be easily applied to other common settings such as the classification problem and the case where there exists an unknown deterministic target hypothesis.

This paper chooses to analyze the BWA algorithm for data that are V-geometrically ergodic. Theoretically, V-geometrically ergodic Markov chains have many good properties that make them appealing for analyses. Firstly, they are ``nice'' general state space Markov chains as they mix geometrically fast \cite{meyn2009markov}. Secondly, the fact that these chains can be defined on a general, possibly uncountable, state space makes their learning models more general than previous models which learn from finite or countable state space Markov chains \cite{gamarnik2003extension}. Thirdly, the V-geometrically ergodic assumption is not too restrictive since it includes all uniformly ergodic chains as well as all ergodic chains on a finite state space \cite{zou2012,xu2012robustness}. Nevertheless, we emphasize that our proof idea can be applied to other types of mixing Markov chains if we have the uniform convergence rate of the empirical loss for these chains.

\section{Related Work}
The BWA algorithm considered in this paper is a batch version of the online weighted average algorithm \cite{kivinen1999averaging}. The main differences are that the BWA algorithm uses an infinite real-valued hypothesis space and is trained from batch data. The original weighted average algorithm is a generalization of the weighted majority algorithm \cite{littlestone1989weighted}. Both algorithms were analyzed for the online setting \cite{kivinen1999averaging,littlestone1989weighted} and a variant of the weighted majority algorithm was analyzed for the classification problem with batched IID data \cite{freund2004generalization}. However, to the best of our knowledge, there was no rigorous treatment for the generalization and robustness of the BWA algorithm for non-IID data.

The proofs in our paper use a previous result on the uniform convergence rate of the empirical loss for V-geometrically ergodic Markov chains \cite{zou2012}. Convergence of the empirical loss is a fundamental problem in statistics and statistical learning theory, and it has been studied for other types of Markov chains such as $\alpha$-mixing \cite{zou2009generalization,vidyasagar2005convergence,zou2007performance}, $\beta$-mixing \cite{yu1994rates,mohri2009rademacher}, $\phi$-mixing \cite{yu1994rates}, and uniformly ergodic \cite{zou2009learning} chains. These results can be used with our proof idea to prove generalization and robustness bounds of the BWA algorithm for those chains.

The robustness of learning algorithms in the presence of noise has been studied for Valiant's PAC model with IID data \cite{kearns1998efficient,blum2003noise,aslam1993general,goldman1995can}. Recently, Xu et al. \cite{xu2012robustness} analyzed the generalization of learning algorithms based on their algorithmic robustness, the ability of an algorithm to achieve similar performances on similar training and testing data. Their analyses hold for both IID and uniformly ergodic Markov data. Another related concept is stability, the ability of an algorithm to return similar hypotheses when small changes are made to the training data \cite{mohri2010stability}. Stability-based generalization bounds of learning algorithms were proven by Mohri et al. for $\phi$-mixing and $\beta$-mixing data \cite{mohri2010stability}. Our bounds, in contrast, are obtained without measuring the algorithmic robustness or stability of the BWA algorithm.

\section{Preliminaries} \label{sec:pre}
We now introduce the V-geometrically ergodic Markov chains and the settings for our analyses. We will follow the definitions in \cite{zou2012}. We also review a result on the uniform convergence rate of the empirical loss for V-geometrically ergodic Markov data \cite{zou2012} which will be used in the subsequent sections.

\subsection{V-geometrically Ergodic Markov Chain}
Let $(\mc{Z}, \mc{F})$ be a measurable space, where $\mc{Z}$ is a compact subset of $\bb{R}^N$ ($N \geq 1$) and $\mc{F}$ is a $\sigma$-algebra on $\mc{Z}$. A Markov chain on $\mc{Z}$ is a sequence of random variables $(Z_i)_{i=1}^\infty$ together with a set of transition probabilities $\{ P^n(A|z) : z \in \mc{Z} \text{ and } A \in \mc{F}\}$, where $P^n(A|z)$ denotes the probability that a chain starting from $z$ will be in $A$ after $n$ steps. By Markov property,
\[ P^n(A|z) = \bb{P}(Z_{n+m} \in A | Z_i, i \leq m, Z_m = z) = \bb{P}(Z_{n+m} \in A | Z_m = z) \]
where $\bb{P}(.)$ is the probability of an event. For any two probability measures $P_1$ and $P_2$ on $(\mc{Z}, \mc{F})$, we define their total variation distance as $\|P_1 - P_2\|_{TV} = 2 \sup_{A \in \mc{F}}|P_1(A) - P_2(A)|$. A V-geometrically ergodic Markov chain can be defined as follows.

\begin{definition} \label{def:vgeo}
A Markov chain $(Z_i)_{i=1}^\infty$ is called V-geometrically ergodic with respect to a measurable function $V : \mc{Z} \to [1,\infty)$ if there exist $\gamma < \infty$, $\rho < 1$, and $B < \infty$ such that for every $z_j,z_k \in \mc{Z}$ and $n \geq 1$, we have
\[ \|P^n(z_j|z_k) - \pi(z_j)\|_{TV} \leq \gamma \rho^n V(z_k) \]
and 
\[ \int_{\mc{Z}}{V(z)\pi(dz) < B} \]
where $\pi$ is the stationary distribution of the Markov chain $(Z_i)_{i=1}^\infty$.
\end{definition}

A special case of V-geometrically ergodic Markov chain is uniformly ergodic Markov chain, which has $V \equiv 1$ (the constant function $1$) \cite{zou2012,meyn2009markov}. So, the results in this paper also hold for the uniformly ergodic Markov data. Throughout our paper, we mostly consider the first $n$ elements $(Z_i)_{i=1}^n$ of a V-geometrically ergodic Markov chain $(Z_i)_{i=1}^\infty$. For convenience, we will also call $(Z_i)_{i=1}^n$ a V-geometrically ergodic Markov chain. Whenever we consider $\pi$, $\gamma$, $\rho$ and $B$ of $(Z_i)_{i=1}^n$, we actually refer to those of $(Z_i)_{i=1}^\infty$.

\subsection{Settings}
\label{sec:settings}
We assume that the training data $(Z_i)_{i=1}^n = (X_i, Y_i)_{i=1}^n$ form a V-geometrically ergodic Markov chain on a state space $\mc{Z} = \mc{X} \times \mc{Y}$, where $\mc{X}$ is a compact subset of $\bb{R}^d$ ($d \geq 1$) and $\mc{Y}$ is a compact subset of $\bb{R}$. The variables $X_i$'s are usually called the observation variables and $Y_i$'s are usually called the target variables.

Let $\mc{H}$ be the set of all hypotheses, where a hypothesis $h$ is a function from $\mc{X}$ to $\mc{Y}$. Throughout this paper, we make the following assumption: $\mc{H}$ is contained in a ball $B(C^q(\mc{X}))$ of a H\"older space $C^q(\mc{X})$ for some $q > 0$, which is similar to the assumption in \cite{zou2012}. The H\"older space $C^q(\mc{X})$ is the space of all continuous functions on $\mc{X}$ with the following norm \cite{zou2012,zhou2003capacity}:
\[
\|h\|_{C^q(\mc{X})} = \|h\|_\infty + \sup_{x_1 \ne x_2; x_1,x_2 \in \mc{X}} \frac{|h(x_1)-h(x_2)|}{\|x_1-x_2\|_{\bb{R}^d}^q}
\]
where $\|h\|_\infty = \sup_{x \in \mc{X}}|h(x)|$ and $\|.\|_{\bb{R}^d}$ is a metric defined on $\bb{R}^d$.

In this paper, we consider the $L_1$-loss $L_1(h,z) = |h(x)-y|$ of a hypothesis $h \in \mc{H}$ on an example $z = (x,y) \in \mc{Z}$. Because of the boundedness of $\mc{X}$ and $\mc{Y}$, there exist $M > 0$ and $L > 0$ such that
\[ M = \sup_{h \in \mc{H}} \max_{(x,y) \in \mc{X} \times \mc{Y}} |h(x)-y| \]
and
\[ L = \sup_{\substack{h_1,h_2 \in \mc{H} \\ h_1 \ne h_2}} \max_{(x,y) \in \mc{X} \times \mc{Y}} \frac{\big | |h_1(x)-y| - |h_2(x)-y| \big |}{\|h_1-h_2\|_\infty}. \]

For any data $S = (X_i, Y_i)_{i=1}^n$, we define the empirical loss of the hypothesis $h$ on $S$ as
\[ l_S(h) = \frac{1}{n}\sum_{i=1}^n |h(X_i) - Y_i| \]
and the expected loss of $h$ with respect to the stationary distribution $\pi$ of the Markov chain as
\[ l(h) = \bb{E}_{(X,Y) \sim \pi}|h(X) - Y|. \]

\subsection{Uniform Convergence Rate of the Empirical Loss}
\label{sec:uni-convergence}
We review a previous result \cite{zou2012} which gives a PAC-style bound on the training set size for the empirical loss to converge uniformly to the expected loss when training data are V-geometrically ergodic Markov chains. This result will be used to prove the generalization and robustness bounds for the BWA algorithm in subsequent sections. To state the result, we first need to define the \emph{covering number}, the quantity for measuring the capacity of a hypothesis space.
\begin{definition}
For every $\epsilon > 0$, the covering number $\mc{N}(\mc{H},\epsilon)$ of the hypothesis space $\mc{H}$ is the smallest integer number $m \in \bb{N}$ such that $\mc{H}$ can be covered by $m$ balls with radius $\epsilon$.
\end{definition}

Note that the covering number $\mc{N} (\mc{H}, \epsilon)$ is defined with respect to the norm $\| \cdot \|_{C^q(\mc{X})}$ and thus is data independent. This is different from another type of covering number which is data dependent \cite{bousquet2004introduction}. With the assumption that $\mc{H} \subseteq B(C^q(\mc{X}))$, there exists $c > 0$ such that for every $\epsilon > 0$, we have $\mc{N}(\mc{H},\epsilon) \leq \exp \{ c \epsilon^{-2d/q} \}$ (see \cite{zhou2003capacity}). Thus, the covering number is finite in our setting.

We also need a concept of effective sample size $n_e$ for a V-geometrically ergodic Markov chain. The effective sample size plays the same role in our analyses as the sample size in the IID case. This concept is usually used when the observations are not independent (e.g., hierarchical autocorrelated observations \cite{ane2008analysis}).
\begin{definition}
Let $S = (X_i, Y_i)_{i=1}^n$ be a V-geometrically ergodic Markov chain with $\rho$ satisfying Definition \ref{def:vgeo}. The effective sample size $n_e$ is
\[ n_e = \left \lfloor \frac{n}{\lceil \{8n/\ln(1/\rho)\}^{1/2} \rceil} \right \rfloor \]
where $\lfloor m \rfloor$ ($\lceil m \rceil$) denote the floor (ceiling) of $m$.
\end{definition}

For a V-geometrically ergodic Markov chain, $n_e \rightarrow \infty$ as $n \rightarrow \infty$. The uniform convergence rate for the empirical loss when training data are V-geometrically ergodic Markov chains is stated in Lemma \ref{thmUnifbd} below. This lemma is a direct consequence of Theorem 2 in \cite{zou2012}.
\begin{lemma}
Let the data $S = (X_i, Y_i)_{i=1}^n$ be a V-geometrically ergodic Markov chain with $\gamma$, $\rho$ and $B$ satisfying Definition \ref{def:vgeo}. For all $\epsilon \in (0,3M]$,  $\delta\in (0,1)$, if the effective sample size $n_e$ satisfies
\[ n_e \geq \frac{8M^2}{\epsilon^2} \left( \ln\frac{2}{\delta} + \ln(1+\gamma B e^{-2}) + \ln \mc{N} \left (\mc{H},\frac{\epsilon}{4L} \right ) \right), \] 
then
\[ \bb{P} \left( \forall h \in \mc{H}, | l_{S}(h) - l(h) | < \epsilon  \right) \geq 1 - \delta. \]
\label{thmUnifbd}
\end{lemma}

\section{The Batched Weighted Average Algorithm}
In this section, we introduce the BWA algorithm. In contrast to the ERM algorithm which makes prediction based on a single empirical loss minimizing hypothesis, the BWA algorithm makes prediction based on the weighted average predictions of all the hypotheses in the hypothesis space. The \emph{pseudo code} for the BWA algorithm is given in Algorithm \ref{algo:average}.

Inputs for the BWA algorithm are a parameter $\alpha < 1$ and a training data sequence $S = (X_i, Y_i)_{i=1}^n$, which is a V-geometrically ergodic Markov chain on the state space $\mc{X} \times \mc{Y}$. The algorithm computes a weight for each hypothesis $h$ in the hypothesis space $\mc{H}$ by:
\[ w_n(h) = \alpha^{n l_S(h)}. \]

Then, the weights of the hypotheses are normalized to obtain a probability density function with respect to the measure $\mu$ (probability mass function if $\mc{H}$ is finite) over the hypothesis space:
\[ P_n(h) = \frac{w_n(h)}{\int_{\mc{H}} w_n(h) d\mu}. \]

We will call $P_n(h)$ the normalized weight of $h$. Given a new example $X$, we use the normalized weights to compute the weighted average prediction of all the hypotheses on $X$:
\[ \ov{h}_n(X) = \int_{\mc{H}} P_n(h) h(X) d\mu. \]

In the algorithm, we assume there exists a probability measure $\mu$ on $\mc{H}$ such that $\mu(\mc{H}) = \int_{\mc{H}} d\mu = 1$. The measure $\mu$ plays a similar role to the prior distribution in Bayesian analysis \cite{mackay1992bayesian}. It reflects our initial belief about the distribution of the hypotheses in $\mc{H}$. During the execution of the algorithm, we gradually update our belief, via the weights, based on the prediction of each hypothesis on the training data. The existence of such a measure $\mu$ was also assumed in \cite{freund2004generalization} for averaged classifiers.

When $\mc{H}$ is infinite, we usually cannot compute the value of $\ov{h}_n$ exactly. In practice, we can apply the Markov Chain Monte Carlo method \cite{brooks1998markov} to approximate $\ov{h}_n$. For instance, we can sample $m$ hypotheses $h_1, h_2, \ldots, h_m$ from the unnormalized density distribution $w_n(h) \mu(h)$ and approximate $\ov{h}_n(X)$ by $\frac{1}{m} \sum_{i=1}^m{h_i(X)}$.

\begin{algorithm}[tb]
\caption{The Batched Weighted Average (BWA) Algorithm}
\label{algo:average}
\begin{algorithmic}
\Require $\alpha < 1$ and training data $(X_i, Y_i)_{i=1}^n$.
\State $w_0(h) \leftarrow 1$ for all $h \in \mc{H}$
  \For{$i = 1 \to n$}
    \For{$h \in \mc{H}$}
       $w_i(h) \leftarrow \alpha^{|h(X_i)-Y_i|} \cdot w_{i-1}(h)$
    \EndFor
  \EndFor
\State $\displaystyle P_n(h) \leftarrow \frac{w_n(h)}{\int_{\mc{H}} w_n(h) d\mu}$ for all $h \in \mc{H}$
\State \Return $\displaystyle \ov{h}_n(X) = \int_{\mc{H}} P_n(h) h(X) d\mu$
\end{algorithmic}
\end{algorithm}

\section{Generalization Bound for BWA Algorithm} \label{sec:AH}
In this section, we prove the generalization bound for the BWA algorithm when training data are V-geometrically ergodic Markov chains. For the analyses to be valid, we assume the following sets are measurable with respect to $\mu$:
\[ \{ h \in \mc{H} : l(h) \leq \epsilon \} \text{, for all } \epsilon \in \bb{R}. \]

Since Algorithm \ref{algo:average} does not assume the existence of a perfect hypothesis in $\mc{H}$, we need to define the optimal expected loss of $\mc{H}$. Let $\mc{H}_\gamma = \{ h \in \mc{H}: l(h) \leq \gamma \}$, the optimal expected loss of $\mc{H}$ is defined as $\gamma^* = \inf\{\gamma: \mu(\mc{H}_\gamma) > 0 \}$. Note that $\gamma^*$ always exists since $\mu(\mc{H}_M) = 1$ and $\{\gamma: \mu(\mc{H}_\gamma) > 0 \} \neq \emptyset$. For all $\epsilon > 0$, let $\mc{V}_{\epsilon}=\mu(\mc{H}_{\gamma^*+\epsilon})$ be the volume of all the hypotheses with expected loss at most $\gamma^* + \epsilon$. By definition of $\gamma^*$, for all $\epsilon > 0$, we always have $\mc{V}_{\epsilon} > 0$.

The idea of using $\mc{V}_\epsilon$ was proposed in \cite{freund2004generalization} to analyze the generalization bounds of averaged classifiers in the IID case. The argument for considering $\mc{V}_{\epsilon}$ is that when $\mc{H}$ is uncountable, a comparison between the average hypothesis $\ov{h}_n$ and a single best hypothesis is meaningless because a single hypothesis mostly has measure $0$. Hence, we should compare $\ov{h}_n$ to a set of good hypotheses that has positive measure, as suggested in \cite{freund2004generalization}.

To prove the generalization bound, we need Lemma \ref{thm01} that bounds the normalized weights $P_n(h)$ of all the bad hypotheses. Specifically, this lemma proves that if the effective sample size is large enough, the normalized weights of all the bad hypotheses are sufficiently small with high probability.

\begin{lemma}
Let the data $S = (X_i, Y_i)_{i=1}^n$ be a V-geometrically ergodic Markov chain with $\gamma$, $\rho$ and $B$ satisfying Definition \ref{def:vgeo}. For all $\epsilon \in (0,3M]$ and $\delta\in (0,1)$, if the effective sample size $n_e$ satisfies
\[ n_e \geq \frac{288M^2}{\epsilon^2} \left( \ln\frac{2}{\delta} + \ln(1+\gamma B e^{-2}) + \ln \mc{N} \left (\mc{H},\frac{\epsilon}{24L} \right ) \right), \] 
then
\[ \bb{P} \left( \sup_{h \in \mc{H} \setminus \mc{H}_{\gamma^* + \epsilon}} P_n(h) \leq \frac{\alpha^{n \epsilon/6}}{\mc{V}_{\epsilon/2}} \right) \geq 1 - \delta. \]
\label{thm01}
\end{lemma}

\begin{proof}
Denote $r_n(h) = \alpha^{l_S(h)} = (w_n(h))^{1/n}$ and $\displaystyle \|r_n\|_n = \left( \int_{\mc{H}}{|r_n(h)|^n d\mu} \right)^{1/n}$. 
We can write: $\displaystyle P_n(h) = \frac{w_n(h)}{\int_{\mc{H}}{w_n(h) d\mu}} = \left( \frac{r_n(h)}{\|r_n\|_n} \right)^n$. 
If the effective sample size satisfies
\[ n_e \geq \frac{288M^2}{\epsilon^2} \left( \ln\frac{2}{\delta} + \ln(1+\gamma B e^{-2}) + \ln \mc{N} \left (\mc{H},\frac{\epsilon}{24L} \right ) \right), \] 
then by Lemma \ref{thmUnifbd}, with probability at least $1 - \delta$, we both have:
\[ |l_S(h) - l(h)| < \epsilon/6 \text{, for all } h \in \mc{H} \setminus \mc{H}_{\gamma^* + \epsilon} \]
\[ |l_S(h') - l(h')| < \epsilon/6 \text{, for all } h' \in \mc{H}_{\gamma^* + \epsilon/2}. \]

For all $h \in \mc{H} \setminus \mc{H}_{\gamma^*+\epsilon}$ and $h' \in \mc{H}_{\gamma^*+\epsilon/2}$, we also have $l(h) - l(h') \geq \epsilon/2$. Therefore, with probability at least $1 - \delta$, for all $h \in \mc{H} \setminus \mc{H}_{\gamma^*+\epsilon}$ and $h' \in \mc{H}_{\gamma^*+\epsilon/2}$,
\[ l_S(h) - l_S(h') \geq \epsilon/6. \]

Since $\alpha < 1$, we have $\alpha^{l_S(h) - l_S(h')} \leq \alpha^{\epsilon/6}$. Hence, $r_n(h) \leq \alpha^{\epsilon/6} r_n(h')$. Note that this inequality holds for all $h \in \mc{H} \setminus \mc{H}_{\gamma^* + \epsilon}$ and $h' \in \mc{H}_{\gamma^* + \epsilon/2}$. Therefore,
\[ \sup_{h \in \mc{H} \setminus \mc{H}_{\gamma^* + \epsilon}} r_n(h) \leq \alpha^{\epsilon/6} \inf_{h' \in \mc{H}_{\gamma^* + \epsilon/2}} r_n(h'). \]

Let $r^* = \inf_{h' \in \mc{H}_{\gamma^* + \epsilon/2}} r_n(h')$, we have
\[
\|r_n\|_n = \left( \int_{\mc{H}}{|r_n(h)|^n d\mu} \right)^{1/n} \geq \left( \int_{\mc{H}_{\gamma^* + \epsilon/2}}{(r^*)^n d\mu} \right)^{1/n} = r^* \mu (\mc{H}_{\gamma^* + \epsilon/2})^{1/n} = r^* \mc{V}_{\epsilon/2}^{1/n}. \]

Therefore, $\displaystyle \sup_{h \in \mc{H} \setminus \mc{H}_{\gamma^* + \epsilon}} P_n(h) = \sup_{h \in \mc{H} \setminus \mc{H}_{\gamma^* + \epsilon}} \left ( \frac{r_n(h)}{\|r_n\|_n} \right )^n \leq \frac{\alpha^{n \epsilon/6}}{\mc{V}_{\epsilon/2}}$.
\end{proof}

Using Lemma \ref{thm01}, we now prove the following generalization bound for the BWA algorithm.
\begin{theorem}
Let the data $S = (X_i, Y_i)_{i=1}^n$ be a V-geometrically ergodic Markov chain with $\gamma$, $\rho$ and $B$ satisfying Definition \ref{def:vgeo}. For all $\epsilon \in (0,3M]$ and $\delta\in (0,1)$, if the effective sample size $n_e$ satisfies
\begin{eqnarray*}
n_e &\geq& \frac{1152M^2}{\epsilon^2} \left( \ln\frac{2}{\delta} + \ln (1+\gamma B e^{-2}) + \ln \mc{N} \left (\mc{H},\frac{\epsilon}{48L} \right ) \right) + \left( \frac{3 \left( \ln \frac{1}{\mc{V}_{\epsilon/4}} + \ln\frac{2M}{\epsilon} \right)}{2 \epsilon \ln \frac{1}{\alpha} \ln \frac{1}{\rho} } \right )^{1/2},
\end{eqnarray*}
then
\[ \bb{P} \left( l(\ov{h}_n) - \gamma^* \leq \epsilon \right) \geq 1 - \delta. \]
\label{thm02}
\end{theorem}

\begin{proof}
We have
\begin{eqnarray*}
&& l(\ov{h}_n) = \bb{E}_{(X,Y) \sim \pi} |\ov{h}_n(X) - Y| = \bb{E}_{(X,Y) \sim \pi} \left| \int_{\mc{H}} P_n(h) h(X) d\mu - Y \right| \\
&\leq& \bb{E}_{(X,Y) \sim \pi} \left[ \int_{\mc{H}} {P_n(h)|h(X) - Y| d\mu} \right] \\
&=& \bb{E}_{(X,Y) \sim \pi} \int_{\mc{H}_{\gamma^*+\epsilon/2}} {P_n(h)|h(X) - Y| d\mu} + \bb{E}_{(X,Y) \sim \pi} \int_{\mc{H} \setminus \mc{H}_{\gamma^*+\epsilon/2}} {P_n(h)|h(X) - Y| d\mu} \\
&=& \int_{\mc{H}_{\gamma^*+\epsilon/2}} P_n(h) \bb{E}_{(X,Y) \sim \pi} |h(X) - Y| d\mu + \int_{\mc{H} \setminus \mc{H}_{\gamma^*+\epsilon/2}} P_n(h) \bb{E}_{(X,Y) \sim \pi} |h(X) - Y| d\mu.
\end{eqnarray*}

Notice that for all $h \in \mc{H}_{\gamma^*+\epsilon/2}$, we have: $\displaystyle \bb{E}_{(X,Y) \sim \pi} |h(X) - Y| \le \gamma^*+ \frac{\epsilon}{2}$.
On the other hand, from Lemma \ref{thm01}, if the effective sample size satisfies
\[ n_e \geq \frac{1152M^2}{\epsilon^2} \left( \ln\frac{2}{\delta} + \ln (1+\gamma B e^{-2}) + \ln \mc{N} \left (\mc{H},\frac{\epsilon}{48L} \right ) \right), \] 
then with probability at least $1 - \delta$, we have: $\displaystyle \sup_{h \in \mc{H} \setminus \mc{H}_{\gamma^* + \epsilon/2}} P_n(h) \leq \frac{\alpha^{n \epsilon/12}}{\mc{V}_{\epsilon/4}}$.

Thus,
\begin{eqnarray*}
l(\ov{h}_n) &\leq&  (\gamma^*+ \frac{\epsilon}{2}) \int_{\mc{H}_{\gamma^*+\epsilon/2}} P_n(h) d\mu + \frac{\alpha^{n \epsilon/12}}{\mc{V}_{\epsilon/4}} \int_{\mc{H} \setminus \mc{H}_{\gamma^*+\epsilon/2}} \bb{E}_{(X,Y) \sim \pi} |h(X) - Y| d\mu \\
&\leq&  (\gamma^*+\frac{\epsilon}{2}) \int_{\mc{H}} P_n(h) d\mu + \frac{\alpha^{n \epsilon/12}}{\mc{V}_{\epsilon/4}} \int_{\mc{H} \setminus \mc{H}_{\gamma^*+\epsilon/2}} M d\mu \\
&\leq&  (\gamma^*+ \frac{\epsilon}{2}) + \frac{\alpha^{n \epsilon/12}}{\mc{V}_{\epsilon/4}} M.
\end{eqnarray*}

Note that when $\displaystyle n \geq \frac{12}{\epsilon \ln \frac{1}{\alpha}} \left( \ln \frac{1}{\mc{V}_{\epsilon/4}} + \ln \frac{2M}{\epsilon} \right)$, we have $\displaystyle \frac{\alpha^{n \epsilon/12}}{\mc{V}_{\epsilon/4}} M \leq \frac{\epsilon}{2}$.
From the definition of the effective sample size, in order to ensure the previous condition for the sample size $n$, it is sufficient to let
\[ n_e \geq \left( \frac{3 \left(\ln \frac{1}{\mc{V}_{\epsilon/4}} + \ln\frac{2M}{\epsilon} \right)}{2 \epsilon \ln \frac{1}{\alpha} \ln \frac{1}{\rho} } \right )^{1/2}. \]

Hence, for 
\begin{eqnarray*}
n_e &\geq& \frac{1152M^2}{\epsilon^2} \left( \ln\frac{2}{\delta} + \ln (1+\gamma B e^{-2}) + \ln \mc{N} \left (\mc{H},\frac{\epsilon}{48L} \right ) \right) + \left( \frac{3 \left( \ln \frac{1}{\mc{V}_{\epsilon/4}} + \ln\frac{2M}{\epsilon} \right)}{2 \epsilon \ln \frac{1}{\alpha} \ln \frac{1}{\rho} } \right )^{1/2},
\end{eqnarray*}
we have $\bb{P} \left( l(\ov{h}_n) \leq \gamma^* + \epsilon \right) \geq 1 - \delta$.
\end{proof}

In Theorem \ref{thm02}, the convergence rate of the expected loss to the optimal loss depends not only on the covering number $\mc{N} \left (\mc{H},\frac{\epsilon}{48L} \right )$ but also on $\mc{V}_{\epsilon/4}$. From the definition of $\mc{V}_{\epsilon/4}$, this value depends mostly on the distribution $\mu$ on $\mc{H}$. If $\mu$ gives higher probability to hypotheses with small expected loss, $\mc{V}_{\epsilon/4}$ will be closer to $1$ and the convergence rate will be better. Thus, it is desirable for the BWA algorithm to choose a good distribution $\mu$. This is analogous to the Bayesian setting where we also need to choose a good prior for the learning algorithm. When $\mc{H}$ is finite, $\mc{V}_{\epsilon/4} = \mu(\mc{H}_{\gamma^*})$ for sufficiently small $\epsilon$. In this case, $\mc{V}_{\epsilon/4}$ does not depend on $\epsilon$, but only depends on $\mu$.

The bound in Theorem \ref{thm02} and all the subsequent bounds depend on the values of $\gamma$, $\rho$ and $B$. For one V-geometrically ergodic Markov chain, there may be many values of ($\gamma$, $\rho$, $B$) satisfying Definition \ref{def:vgeo}. Thus, to obtain good bounds, we need to choose a value of ($\gamma$, $\rho$, $B$) that makes the bounds as tight as possible. This corresponds to selecting small values for these parameters.

When comparing various V-geometrically ergodic Markov chains, Theorem \ref{thm02} suggests that the convergence rate is better if $\gamma$, $\rho$ and $B$ are smaller. Small values of these parameters correspond to chains that converge quickly to the stationary distribution $\pi$. This result is expected because the expected loss $l(\cdot)$ is defined with respect to a random example drawn from $\pi$. In the limit when $\gamma \rightarrow 0$ and $\rho \rightarrow 0$, the chains become more IID-like and the effective sample size bound tends to $\frac{1152M^2}{\epsilon^2} ( \ln\frac{2}{\delta} + \ln \mc{N} (\mc{H},\frac{\epsilon}{48L}))$. 

From the discussion in Section \ref{sec:uni-convergence}, there exists $c > 0$ such that for $\epsilon > 0$, we have $\mc{N}(\mc{H},\epsilon) \leq \exp \{ c \epsilon^{-2d/q} \}$. Therefore, we can deduce the following corollary of Theorem \ref{thm02} in which the bound does not depend on the covering number.

\begin{corollary}
Let the data $S = (X_i, Y_i)_{i=1}^n$ be a V-geometrically ergodic Markov chain with $\gamma$, $\rho$ and $B$ satisfying Definition \ref{def:vgeo}. For all $\epsilon \in (0,3M]$ and  $\delta\in (0,1)$, if the effective sample size $n_e$ satisfies
\begin{eqnarray*}
n_e &\geq& \frac{1152M^2}{\epsilon^2} \left( \ln\frac{2}{\delta} + \ln (1+\gamma B e^{-2}) + c (\frac{\epsilon}{48L})^{-2d/q} \right) + \left(  \frac{3( \ln \frac{1}{\mc{V}_{\epsilon/4}} + \ln\frac{2M}{\epsilon})}{2 \epsilon \ln \frac{1}{\alpha} \ln \frac{1}{\rho} } \right )^{1/2},
\end{eqnarray*}
then $\displaystyle \,\,\, \bb{P} \left( l(\ov{h}_n) - \gamma^* \leq \epsilon \right) \geq 1 - \delta$.
\label{cor01}
\end{corollary}

Since $n_e \rightarrow \infty$ as $n \rightarrow \infty$, by the above corollary, we have $\bb{P} \left( l(\ov{h}_n) - \gamma^* \leq \epsilon \right) \to 1$ for every $\epsilon \in (0,3M]$. Hence, the BWA algorithm is consistent.

\section{Robustness Bound for BWA Algorithm} \label{sec:robustness}
In this section, we consider the robustness of the BWA algorithm when the target variable's values in the training data contain a small amount of noise. In particular, instead of the settings in Section \ref{sec:settings}, we assume that the training data are now $(\wt{Z}_i)_{i=1}^n = (X_i, \wt{Y}_i)_{i=1}^n = (X_i, Y_i + \xi_i)_{i=1}^n$, where $\wt{Y}_i = Y_i + \xi_i$ and $(X_i, Y_i)_{i=1}^n$ form a V-geometrically ergodic Markov chain with stationary distribution $\pi$. We further assume that the noise are bounded, i.e., $-\Xi/2 \leq \xi_i \leq \Xi/2$ for all $i$. However, we will not make any assumption on the distribution of noise.

With this setting, the BWA algorithm that we consider is essentially the same as Algorithm \ref{algo:average}, except that now the algorithm does not have access to the true target variables $Y_i$'s. Instead, it uses the noisy target variables $\wt{Y}_i$ and updates the hypothesis weights according to the following formula:
\[ w_i(h) \leftarrow \alpha^{|h(X_i)-\wt{Y}_i|} \cdot w_{i-1}(h). \]
Hence, $w_n(h) = \alpha^{n l_{\wt{S}}(h)}$, where $l_{\wt{S}}(h)$ is the (noisy) empirical loss of the hypothesis $h$ on the noisy dataset $\wt{S} = (X_i, \wt{Y}_i)_{i=1}^n$:
\[ l_{\wt{S}}(h) = \frac{1}{n}\sum_{i=1}^n |h(X_i) - \wt{Y}_i| = \frac{1}{n}\sum_{i=1}^n |h(X_i) - Y_i - \xi_i|. \]

For any hypothesis $h$, the expected loss $l(h)$ is defined as in Section \ref{sec:settings} with respect to the stationary distribution $\pi$ of the Markov chain $(X_i, Y_i)_{i=1}^n$. We also let $\gamma$, $\rho$ and $B$ be the parameters satisfying Definition \ref{def:vgeo} for the chain $(X_i, Y_i)_{i=1}^n$. The optimal expected loss $\gamma^*$ is defined as in Section \ref{sec:AH}.

We now prove that with this setting, the generalization bound of the BWA algorithm deviates at most by $\Xi$. The steps for the proof are similar to those in Section \ref{sec:AH}. First, we prove the following uniform convergence bound for V-geometrically ergodic Markov chain with bounded noise.
\begin{lemma}
Let the data $\wt{S} = (X_i, \wt{Y}_i)_{i=1}^n = (X_i, Y_i + \xi_i)_{i=1}^n$ be a V-geometrically ergodic Markov chain with bounded noise. For all $\epsilon \in (0,3M]$,  $\delta\in (0,1)$, if the effective sample size $n_e$ satisfies
\[ n_e \geq \frac{8M^2}{\epsilon^2} \left( \ln\frac{2}{\delta} + \ln(1+\gamma B e^{-2}) + \ln \mc{N} \left (\mc{H},\frac{\epsilon}{4L} \right ) \right), \] 
then $\displaystyle \,\,\, \bb{P} \left( \forall h \in \mc{H}, | l_{\wt{S}}(h) - l(h) | < \epsilon + \frac{\Xi}{2}   \right) \geq 1 - \delta$.
\label{thmUnifbd-noisy}
\end{lemma}

\begin{proof}
Let $S = (X_i, Y_i)_{i=1}^n$ and $l_S(h)$ be defined as in Section \ref{sec:settings}. For all $h$,
\begin{eqnarray*}
& & | l_{\wt{S}}(h) - l_S(h) | 
= \frac{1}{n} \left| \sum_{i=1}^n \left( |h(X_i) - Y_i - \xi_i| - |h(X_i) - Y_i| \right) \right| \\
&\leq& \frac{1}{n} \sum_{i=1}^n \left| |h(X_i) - Y_i - \xi_i| - |h(X_i) - Y_i| \right|
\leq \frac{1}{n} \sum_{i=1}^n | \xi_i | \leq \frac{\Xi}{2}.
\end{eqnarray*}
By Lemma \ref{thmUnifbd}, if the effective sample size $n_e$ satisfies
\[ n_e \geq \frac{8M^2}{\epsilon^2} \left( \ln\frac{2}{\delta} + \ln(1+\gamma B e^{-2}) + \ln \mc{N} \left (\mc{H},\frac{\epsilon}{4L} \right ) \right), \]
then $\bb{P} \left( \forall h \in \mc{H}, | l_S(h) - l(h) | < \epsilon \right) \geq 1 - \delta$. In this case,
$| l_{\wt{S}}(h) - l(h) | \leq | l_{\wt{S}}(h) -  l_S(h) | + | l_S(h) - l(h) | < \frac{\Xi}{2} + \epsilon$. Hence, Lemma \ref{thmUnifbd-noisy} holds.
\end{proof}

Using Lemma \ref{thmUnifbd-noisy}, we can prove the following lemma, which is an analogy of Lemma \ref{thm01}.
\begin{lemma}
Let the data $\wt{S} = (X_i, \wt{Y}_i)_{i=1}^n = (X_i, Y_i + \xi_i)_{i=1}^n$ be a V-geometrically ergodic Markov chain with bounded noise. For all $\epsilon \in (0,3M]$ and $\delta\in (0,1)$, if the effective sample size $n_e$ satisfies
\[ n_e \geq \frac{288M^2}{\epsilon^2} \left( \ln\frac{2}{\delta} + \ln(1+\gamma B e^{-2}) + \ln \mc{N} \left (\mc{H},\frac{\epsilon}{24L} \right ) \right), \] 
then $\displaystyle \,\,\, \bb{P} \left( \sup_{h \in \mc{H} \setminus \mc{H}_{\gamma^* + \epsilon + \Xi}} P_n(h) \leq \frac{\alpha^{n \epsilon/6}}{\mc{V}_{\epsilon/2}} \right) \geq 1 - \delta$.
\label{thm01-noisy}
\end{lemma}

\begin{proof}
The proof for this lemma uses the same technique as that of Lemma \ref{thm01}, except that we define $r_n(h) = \alpha^{l_{\wt{S}}(h)}$ and replace Lemma \ref{thmUnifbd} by Lemma \ref{thmUnifbd-noisy} with all $h \in \mc{H} \setminus \mc{H}_{\gamma^*+\epsilon + \Xi}$ and $h' \in \mc{H}_{\gamma^*+\epsilon/2}$.
\end{proof}

Using Lemma \ref{thm01-noisy}, we can prove the following robustness bound.
\begin{theorem}
Let the data $\wt{S} = (X_i, \wt{Y}_i)_{i=1}^n = (X_i, Y_i + \xi_i)_{i=1}^n$ be a V-geometrically ergodic Markov chain with bounded noise. For all $\epsilon \in (0,3M]$ and $\delta\in (0,1)$, if the effective sample size $n_e$ satisfies
\begin{eqnarray*}
n_e &\geq& \frac{1152M^2}{\epsilon^2} \left( \ln\frac{2}{\delta} + \ln (1+\gamma B e^{-2}) + \ln \mc{N} \left (\mc{H},\frac{\epsilon}{48L} \right ) \right) + \left( \frac{3 \left( \ln \frac{1}{\mc{V}_{\epsilon/4}} + \ln\frac{2M}{\epsilon} \right)}{2 \epsilon \ln \frac{1}{\alpha} \ln \frac{1}{\rho} } \right )^{1/2},
\end{eqnarray*}
then $\displaystyle \,\,\, \bb{P} \left( l(\ov{h}_n) - \gamma^* \leq \epsilon + \Xi \right) \geq 1 - \delta$.
\label{thm02-noisy}
\end{theorem}

\begin{proof}
The proof for this theorem is essentially the same as that of Theorem \ref{thm02}, except that we partition $\mathcal{H}$ into $\mc{H}_{\gamma^*+\epsilon/2 + \Xi}$ and $\mc{H} \setminus \mc{H}_{\gamma^*+\epsilon/2 + \Xi}$ after the first inequality and then apply Lemma \ref{thm01-noisy} instead of Lemma \ref{thm01}.
\end{proof}

From Theorem \ref{thm02-noisy}, with high probability, the expected loss of $\ov{h}_n$ is at most $\epsilon + \Xi$ larger than the optimal loss when we allow noise with range $\Xi$ in the training data. This shows that the BWA algorithm is robust in the sense that it does not perform too badly if the level of noise in the training data is small. In the noiseless case where $\Xi = 0$, we can recover Theorem \ref{thm02}. Thus, Theorem \ref{thm02-noisy} is a generalization of Theorem \ref{thm02} to the bounded noise case.

\section{Applications to other Settings}
Our results in Section \ref{sec:AH} and \ref{sec:robustness} are proven for the regression problem when the pairs of observation and target variables are V-geometrically ergodic. We now prove that our results can be easily applied to other common settings such as the classification problem and the case where there exists an unknown deterministic target hypothesis. The discussion in Section \ref{sec:classification} is for the noiseless training data, while the discussion in Section \ref{sec:target} can be applied to both the noiseless and noisy cases. In this section, we let $\mb{1}_A$ be the indicator function for the event $A$.

\subsection{The Classification Problem}
\label{sec:classification}
For the classification problem, the training data $S = (X_i, Y_i)_{i=1}^n$ satisfy $Y_i \in \{ 0, 1 \}$ for $i = 1, 2, \ldots, n$; and during testing, we need to predict the label $Y \in \{ 0, 1\}$ of a given data point $X$. If the hypothesis space $\mc{H}$ contains the hypotheses $h$ satisfying $h(X') = \bb{P}(Y'=1 | X', h)$ for all $X' \in \mc{X}$, we can apply Algorithm \ref{algo:average} to compute $\ov{h}_n$ and use its value to construct the following random classifier:
\[ c_n(X) = 
  \begin{cases}
   1 &\text{with probability} ~ \ov{h}_n(X)\\
   0 &\text{with probability} ~ 1 - \ov{h}_n(X).
  \end{cases}
\]
Let $\varepsilon(c_n) = \bb{P}_{(X,Y) \sim \pi} \left( c_n(X) \ne Y \right)$ be the expected error of $c_n$. The following lemma shows that $\varepsilon(c_n)$ is equal to the expected loss of $\ov{h}_n$. Thus, we can bound the probability $\bb{P} \left( \varepsilon(c_n) - \gamma^* \leq \epsilon \right)$ using this lemma and Theorem \ref{thm02}.

\begin{lemma}
For all $n \geq 1$, we have $\varepsilon(c_n) = l(\ov{h}_n)$.
\label{lem03}
\end{lemma}

\begin{proof}
Note that $\bb{P}(c_n(X) \neq Y|X,Y) \sim \text{Bernoulli}(|\ov{h}_n(X)-Y|)$. Thus,
\begin{eqnarray*}
\varepsilon(c_n) &=& \bb{P}_{(X,Y) \sim \pi} \left( c_n(X) \ne Y \right) = \bb{E}_{(X,Y) \sim \pi} \left[ \mb{1}_{c_n(X) \ne Y } \right]  \\
&=& \bb{E}_{(X,Y) \sim \pi} \left [ \bb{E} \left [ \mb{1}_{c_n(X) \ne Y } | X,Y \right ] \right ] = \bb{E}_{(X,Y) \sim \pi}|\ov{h}_n(X) - Y| = l(\ov{h}_n).
\end{eqnarray*}
\end{proof}

\subsection{When a Target Hypothesis Exists}
\label{sec:target}
When there exists an unknown deterministic target hypothesis $c : \mc{X} \to \mc{Y}$ such that $Y_i = c(X_i)$ for all $i = 1, 2, \ldots, n$ and the observation variables $(X_i)_{i=1}^n$ form a V-geometrically ergodic Markov chain, the following lemma shows that the chain $(X_i, c(X_i))_{i=1}^n$ is V-geometrically ergodic. Thus, our previous results can still be applied in this situation. Note that in our lemma, $c$ may not be in $\mc{H}$.

\begin{lemma}
Let $V_X : \mc{X} \to [1,\infty)$ be a measurable function and $(X_i)_{i=1}^n$ be a $V_X$-geometrically ergodic Markov chain on $\mc{X}$. For any deterministic function $c : \mc{X} \to \mc{Y}$, the chain $(X_i, Y_i)_{i=1}^n = (X_i, c(X_i))_{i=1}^n$ is a V-geometrically ergodic Markov chain on $\mc{X} \times \mc{Y}$ with respect to some measurable function $V : \mc{X} \times \mc{Y} \to [1,\infty)$.
\label{lem02}
\end{lemma}

\begin{proof}
Let $P_X$ be the one-step transition probability of $(X_i)_{i=1}^n$. It is easy to see that $(X_i, Y_i)_{i=1}^n = (X_i, c(X_i))_{i=1}^n$ is a Markov chain on $\mc{X} \times \mc{Y}$ with the following one-step transition probability $P$:
\begin{eqnarray*}
P(x', y' | x, y) &=&  
\begin{cases} 
P_X(x' | x) \cdot \mb{1}_{ y' = c(x') } & \mbox{if } y = c(x) \\
\mb{1}_{ y' = c(x) } \cdot \mb{1}_{ x' = x } & \mbox{if } y \ne c(x).
\end{cases}
\end{eqnarray*}
Intuitively, after taking the first step (from $(X_1, Y_1)$ onwards), the new Markov chain on $\mc{X} \times \mc{Y}$ will transit around the points in $\{ (x, c(x)) : x \in \mc{X} \}$ with the same probabilities as the transitions on $\mc{X}$. Thus, the new Markov chain has the stationary distribution $\pi(x,y) = \pi_X(x) \cdot \mb{1}_{ y = c(x) }$, where $\pi_X$ is the stationary distribution of $(X_i)_{i=1}^n$. 
Let $\gamma$, $\rho$, and $B$ be the parameters satisfying Definition \ref{def:vgeo} for the chain $(X_i)_{i=1}^n$ and consider the measurable function $V$ as follows:
\begin{eqnarray*}
V(x,y) &=& 
\begin{cases} 
V_X(x)/\rho & \mbox{if } y \ne c(x) \\ 
V_X(x) & \mbox{if } y = c(x).
\end{cases}
\end{eqnarray*}

We have $\int_{\mc{X} \times \mc{Y}} V(x,y)\pi(x,y)d(x,y) = \int_{\mc{X}} V_X(x)\pi_X(x)dx < B$. Furthermore, for any two points $(x,y)$ and $(x',y')$ in $\mc{X} \times \mc{Y}$, the n-step transition probability from $(x,y)$ to $(x',y')$ satisfies:
\begin{eqnarray*}
P^n(x', y' | x, y) &=&  
\begin{cases} 
P^n_X(x' | x) \cdot \mb{1}_{ y' = c(x') } & \mbox{if } y = c(x) \\
P^{n-1}_X(x' | x) \cdot \mb{1}_{ y' = c(x') } & \mbox{if } y \ne c(x).
\end{cases}
\end{eqnarray*}

Thus, for all $n \geq 1$, we have: $\|P^n(x', y' | x, y) - \pi(x', y')\|_{TV} \leq \gamma \rho^n V(x,y)$. Hence, $(X_i, Y_i)_{i=1}^n$ satisfies the V-geometrically ergodic definition with the same parameters $\gamma$, $\rho$, $B$ and the function $V$ above.
\end{proof}

\section{Conclusion} \label{sec:con}
A good property of the BWA algorithm is that the normalized weights of the good hypotheses will eventually dominate those of the bad ones when more training data are obtained. This property enables us to obtain its generalization and robustness bounds for V-geometrically ergodic Markov data. The bounds can be applied to various settings such as the regression problem, the classification problem, and the case where there exists a deterministic target hypothesis. Our results show that the BWA algorithm is consistent and robust for V-geometrically ergodic Markov data. So, when overfitting is involved or when optimizing the empirical risk is hard, it may be a good replacement for the ERM algorithm.

\bibliographystyle{splncs}
\bibliography{llncs}

\end{document}